\documentclass[letterpaper]{article} 
\usepackage{aaai24}  
\usepackage{times}  
\usepackage{helvet}  
\usepackage{courier}  
\usepackage[hyphens]{url}  
\usepackage{graphicx} 
\usepackage{natbib}  
\usepackage{caption} 
\usepackage{algorithm}
\usepackage{algorithmic}
\usepackage{amsthm}
\usepackage{amssymb}
\usepackage{newfloat}
\usepackage{listings}
\DeclareCaptionStyle{ruled}{labelfont=normalfont,labelsep=colon,strut=off} 
\floatstyle{ruled}
\newfloat{listing}{tb}{lst}{}
\floatname{listing}{Listing}
\usepackage{algorithm}
\usepackage{algorithmic}
\usepackage{graphicx}
\usepackage{subfigure}
\newtheorem{theorem}{Theorem}
\newtheorem{assumption}{Assumption}
\newtheorem{corollary}{Corollary}
\newtheorem{lemma}{Lemma}
\newtheorem{remark}{Remark}
\usepackage{tikz}
\usepackage{amsmath}
\title{Limited Memory Online Gradient Descent for Kernelized Pairwise Learning \\with Dynamic Averaging}
\author {
    Hilal AlQuabeh\textsuperscript{\rm1},
    William de Vazelhes\textsuperscript{\rm1},
    Bin Gu\textsuperscript{\rm1,2}\thanks{Corresponding author.}
}
\affiliations {
    \textsuperscript{\rm 1}Mohamed bin Zayed University of Artificial Intelligence, UAE\\
    \textsuperscript{\rm 2}School of Artificial Intelligence, Jilin University, China\\
    hilal.alquabeh@mbzuai.ac.ae, wdevazelhes@gmail.com, jsgubin@gmail.com
}

\usepackage{bibentry}
\begin{document}

\maketitle

\begin{abstract}
Pairwise learning, an important domain within machine learning, addresses loss functions defined on pairs of training examples, including those in metric learning and AUC maximization. 
Acknowledging the quadratic growth in computation complexity accompanying pairwise loss as the sample size grows, researchers have turned to online gradient descent (OGD) methods for enhanced scalability.
Recently, an OGD algorithm emerged, employing gradient computation involving prior and most recent examples, a step that effectively reduces algorithmic complexity to $O(T)$, with $T$ being the number of received examples. This approach, however, confines itself to linear models while assuming the independence of example arrivals.
We introduce a lightweight OGD algorithm that does not require the independence of examples and generalizes to kernel pairwise learning.  Our algorithm builds the gradient based on a random example and a moving average representing the past data, which results in a sub-linear regret bound with a complexity of $O(T)$. Furthermore, through the integration of $O(\sqrt{T}{\log{T}})$ random Fourier features, the complexity of kernel calculations is effectively minimized.  
Several experiments with real-world datasets show that the proposed technique outperforms kernel and linear algorithms in offline and online scenarios. 
\end{abstract}

\section{Introduction}\label{sec:intro}
The concept of online pairwise learning has attracted significant attention due to its unique characteristics and potential applications. Unlike traditional machine learning, online pairwise learning delves into understanding connections between pairs of entities or data points, rather than isolating individual instances. This methodology holds particular relevance in domains where relative evaluations between entities hold substantial import, as seen in Area Under the ROC Curve (AUC) maximization (e.g. in unbalanced-data classification)~\citep{zhang2016pairwise}, metric learning (e.g. in image retrieval)~\citep{kulis2012metric}, and bipartite ranking~\cite{du2016online}. Given that the loss functions in these problems require pairs of examples, a distinctive analysis is needed when compared to other machine learning problems.

At the core of online pairwise learning lies the challenge of efficiently updating models in real-time. This challenge emerges from the cumulative influence of previous instances on the comprehension of new pairs. In this landscape, Online Gradient Descent (OGD) stands as a significant tool within the realm of online learning. OGD takes on a distinctive role within the sphere of online pairwise learning, entailing a nuanced examination of evolving loss gradients and the impact of prior examples with each incoming instance. This nuanced viewpoint results in a computational complexity translating to $O(T^2)$ of gradient computations.

OGD has garnered considerable attention in the field of online pairwise learning, leading to the development of various methodologies. These approaches include online buffer learning \citep{zhao2011online, kar13}, second-order statistic online learning \citep{gao2013one}, and methods rooted in solving the saddle point problem \citep{ying2016stochastic, reddi2016stochastic}, all of which have predominantly harnessed linear models. However, exploration into non-linear models within this domain has been somewhat limited, with kernelized learning receiving minimal attention \citep{ying2015online, du2016online}.

A notable avenue in the realm of pairwise learning is online buffer learning, as introduced by Zhao et al. \citep{zhao2011online}. This technique employs a finite buffer alongside reservoir sampling to mitigate time complexity, reducing it to $O(sT)$, where $s$ signifies the buffer size. By retaining a subset of data and ensuring uniform sampling within the buffer, this strategy effectively mitigates the computational load. Moreover, Yang et al. \citep{yang2021simple} achieved optimal generalization with a buffer size of $s=1$, all while assuming the sequential data to be iid instances (e.g., stochastic environments). This accomplishment marks a significant milestone in the study of pairwise learning.

\begin{table*}[t] 
 \centering
\begin{tabular}{llllccc}
\hline
\textbf{Algorithm} 
& \textbf{Problem} & \textbf{Model} & \textbf{Scheme} & \textbf{iid} & \textbf{Time Complexity} & \textbf{Space Complexity} \\ \hline

\citep{pmlr-v80-natole18a}
& AUC              & Linear       & Online         & Yes  & $O(T)$         & $O(d)$                      \\

\citep{ying2016stochastic}
& AUC              & Linear       & Online         & Yes  & $O(T)$         & $O(d)$                      \\

 \citep{zhao2011online} 
& AUC              & Linear         & Online            &  Yes & $O(sT)$     & $O(ds)$           \\
 \citep{gao2013one} 
& AUC              & Linear         & Online            &  Yes & $O(T)$     & $O(d^2)$           \\

\citep{gu2019scalable}             
& AUC          & Linear        & Offline           &  Yes   &       $O(T)$      & $O(d)$         \\
    \citep{yang2021simple}             
& General          & Linear        & Online           &  Yes  &        $O(T)$     & $O(d)$          \\

    \citep{lin2017online}             
& General          & kernel        & Online           &   No   &      $O(T^2)$      & $O(dT)$       \\

 \citep{kakkar2017sparse}             
& AUC          & Kernel        & Offline           &   Yes   &      $O(T\log{T})$    & $O(dT^2)$         \\
\citep{kar13}             
& General          & Linear        & Online           &   Yes   &      $O(sT)$    & $O(ds)$         \\
\hline
AOGD           
& General         & Kernel        & Online       & No    & $O(\frac{D}{d}T)$        & $O(D)$                     \\ \hline
\end{tabular}
 \caption{Recent pairwise learning algorithms ($T$ is the iteration number, $d$: the dimension, $D$ is random features, and $s$ is a buffer size), note the time complexity is w.r.t. gradients computations.}
 \label{table:comp}
\end{table*}

The prevailing frameworks in existing literature have predominantly concentrated on data that is linearly separable, disregarding the complexities linked with non-linear pairwise learning. Furthermore, the online buffer techniques introduced thus far have not adequately tackled the sensitivity of generalization to the non-iid data, particularly in adversarial environments, constraining their capacity to grasp the intricacies inherent in real-world scenarios.

Furthermore, the exploration of non-linear pairwise learning, particularly in the realm of kernel approximation, has been relatively limited. While non-linear methods offer enhanced expressive capability, the computational expense attributed to kernel calculation – scaling as $O(T^2)$ \citep{lin2017online, kakkar2017sparse} – poses hurdles to their practical applicability in terms of scalability and efficiency.

Transitioning to the examination of generalization bounds, prior research extensively investigated online pairwise gradient descent with buffers and linear models \citep{wang2012generalization, kar13}. These studies establish a bound of $O(1/s + 1/\sqrt{T})$ for this approach. However, it's imperative to acknowledge that this bound attains optimality primarily when the buffer size $s$ approximates $O(\sqrt{T})$. This aspect poses challenges in scenarios where a smaller buffer size is preferred.

Collectively, these vulnerabilities underscore the necessity for further research and advancement in the field of online pairwise learning. This effort should encompass addressing the limitations of iid frameworks, conducting a more comprehensive exploration of non-linear methods, and surmounting the computational obstacles tied to kernel computation. Our method expands the scope of online pairwise learning to accommodate nonlinear data by incorporating kernelization into the input space. We tackle the ramifications of non-iid data on regret by evaluating the gradient of both the moving average and a random sample drawn from the past history. Leveraging random Fourier features, we efficiently estimate the kernel while maintaining a sublinear error bound, thus achieving computational efficiency without compromising performance. Through the integration of kernelization, kernel approximation, and dynamic averaging, our approach transcends linear constraints, effectively addresses non-iid data challenges, and mitigates kernel complexity. As a result, we establish a robust online pairwise learning paradigm (as illustrated in Table \ref{table:comp}). Our principal contributions can be summarized as follows:
\begin{itemize}
    \item We introduce an online pairwise algorithm for non-linear models, characterized by the ability to handle non-iid data. In scenarios featuring a wide margin in the kernel space, our algorithm attains sublinear regret with a buffer size of $O(1)$.
    \item We confront the influence of non-iid data on regret and introduce a random Bernoulli sampling strategy to manage and enhance the regret rate.
    \item In the case of Gaussian kernel, we approximate the pairwise kernel function utilizing a mere $O(\sqrt{T}\log{T})$ features, a significant reduction compared to $O(T)$ features in literature, while keeping a sublinear regret bound.
    \item We validate the effectiveness of our proposed methodology across a range of real-world datasets and compare its performance against state-of-the-art methods for AUC maximization. 
\end{itemize}
The subsequent sections are structured as follows: Section 2 delineates the problem setting, Section 3 expounds on our proposed method, Section 4 furnishes the analysis of regret, Section 5 delves into related work, followed by Section 6 presents experimental results, and lastly, Section 7 concludes the paper.

\section{Problem Setting}
The notion of pairwise learning pertains to a specific subset $\mathcal{X} \subset \mathbb{R}^d$ and a label space $\mathcal{Y} \subset \mathbb{R}$. This concept branches into two categories, as discussed in more detail by \cite{alquabeh2022pairwise}. Our focus lies in analyzing pairwise loss functions from both branches, connecting the pairwise kernel associated with pairwise hypotheses to regular kernels. This link allows us to explore the characteristics of pairwise loss functions within the regular kernel framework. In our study, we examine an algorithmic process that learns from a set of instances denoted as ${z_i := (x_i, y_i) \in \mathcal{Z} := \mathcal{X} \times \mathcal{Y}}$, where the index $i$ ranges from $1$ to $T$, encompassing the total number of received examples. Assuming that the function $f$ belongs to a specific space denoted as $\mathcal{H}$, we observe that the pairwise loss function plays a crucial role in evaluating performance. This function is denoted as $\ell: \mathcal{H} \times \mathcal{Z}^2 \rightarrow \mathbb{R}+$, signifying that it takes pairs of examples from the combined domain $\mathcal{Z}^2$, and yields values within the positive real numbers $\mathbb{R}+$.

In online learning with pairwise losses, upon receiving a new data point $z_t$, a localized error is used as a performance measure \citep{zhang2016pairwise,boissier2016fast}. This error represents the pairing of the new data point with all preceding $t-1$ points. The properties of this localized error depend on the selected pairwise loss function, illustrated as:
\begin{equation}\label{eq:localloss}
{L}_t(w_{t-1}) = \frac{1}{t-1} \sum_{i=1}^{t-1} \ell(w_{t-1},z_t,z_i)
\end{equation}
where $w\in \mathcal{H}$ represents a linear function mapping, i.e. $f_w(x) = \langle w,x \rangle$.
To tackle the memory constraints, \cite{zhao2011online} introduced a buffer-based mechanism for local error computation denoted as $\bar{L}_t(w_{t-1})$, as explicated in Equation \ref{eq:oneregret}. At each iteration $t$, the buffer, designated as $B_t$, accommodates a finite set of historical-examples indices, and the size of this buffer is $|B_t|$ (or $s$ in literature).

\begin{equation}\label{eq:oneregret}
\bar{L}_t(w_{t-1})= \frac{1}{|B_t|}\sum_{i\in B_t}\ell(w_{t-1},z_t,z_i)
\end{equation}

This buffer has an important role in the learning journey, its contents updated at each step through diverse strategies that span from randomized methodologies such as reservoir sampling \citep{zhao2011online, kar13} to non-randomized approaches exemplified by FIFO \citep{yang2021simple}. However, despite the wide adoption of these sampling techniques, there exists a discernible research void in understanding their implications within the context of non-iid online examples. It's imperative to highlight that the issue of buffer size and its update methods becomes pronounced when dealing with adversarial environment scenarios. The choice of buffer size and sampling strategy can significantly impact the model's performance and ability to generalize in the case of non-iid data streams. 

In our pursuit of addressing the intricacies inherent in complex real-world data, our online pairwise approach takes into consideration the mapping of both the hypothesis and the data to a Reproducing Kernel Hilbert Space (RKHS), denoted as $\mathcal{H}$. The associated Mercer pairwise kernel function, denoted as $k:\mathcal{X}^4\mapsto \mathbb{R}$, adheres to the reproducing property $\langle k_{(x,x')}, g \rangle = g(x,x')$, where $x, x' \in \mathcal{X}^2$ and $g \in \mathcal{H}$. Notably, for scenarios involving pointwise hypotheses with pairwise losses, such as AUC loss, the kernel function is streamlined to $k:\mathcal{X}^2\mapsto \mathbb{R}$. The expanse of $\mathcal{H}$ encompasses a spectrum of linear combinations encompassing functional mappings ${k_{(x,x')}|(x,x')\in \mathcal{X}^2}$ and their limit points.

In our quest to circumvent the computational intricacies associated with kernelization in the online realm, we employ the resourceful apparatus of random Fourier features (RFF) as a pragmatic approximation for the Mercer kernel function. RFF introduces a lower-dimensional mapping $r(\cdot)$, serving as an approximation of the kernel function, denoted as $\bar{k}(\cdot)$. This approximation empowers us to execute computations through linear operations, culminating in a significant reduction in computational overhead. The domain defined by the novel kernel functions finds its place in the realm of $\bar{\mathcal{H}}$. It's worth mentioning that earlier investigations have explored the precision of random Fourier approximation within pointwise and offline settings. As we navigate the online landscape, the minimum number of random features necessary to ensure sublinear regret has been ascertained to be $O(T)$. Our method introduces a novel error bound for pairwise predicaments, hinging on the utilization of merely $O(\sqrt{T}\log{T})$ random features (elucidated in Section 5 for comprehensive details).

\subsection{Assumptions}
Prior to unveiling our principal theorems, we lay down a foundation comprising a collection of widely embraced assumptions pertaining to the attributes of both the loss function and kernels.
\begin{assumption}[Lipschitz Continuity]\label{ass:Lipschitz continuous}
Assume for any $z,z'\in\mathcal{Z}$, the loss function $\ell(\cdot,z,z')$ is G-Lipschitz continuous, i.e. $\forall w,w' \in \mathcal{H}$,
\begin{align}
    \nonumber |\ell(w,z,z')-\ell(w',z,z')|\leq G\left\|w-w'\right\|_2.
\end{align}
\end{assumption}

\begin{assumption}[M-smoothness ]\label{ass:smmoth}
Assume the gradient of the loss function $\ell(\cdot,z,z')$ is M-Lipschitz continuous, i.e. $\forall w,w' \in \mathcal{H}$ and for any $z,z'\in\mathcal{Z}$,
\begin{align}
    \nonumber \ell(w) \geq \ell(w') &+ \nabla \ell(w')^T(w-w')      \\\nonumber &+(2M)^{-1}\|\nabla \ell(w) - \nabla \ell(w')\|^2
\end{align}
\end{assumption}

\begin{assumption}[Convexity]\label{ass:Convexity} 
Assume for any $z,z'\in\mathcal{Z}$, the loss function $\ell(\cdot,z,z')$ is convex function, i.e. $\forall w,w' \in \mathcal{H}$,
\begin{align}
    \nonumber \ell(w,z,z') \geq \ell(w',z,z') + \nabla \ell(w',z,z')^T (w - w').
\end{align}
\end{assumption}

\begin{assumption}[Finite Kernel]\label{ass:Kernel}
Assume for any $\rho$-probability measure on $\mathcal{X}^2$ the positive kernel function $k:\mathcal{X}^2 \times \mathcal{X}^2 \rightarrow \mathbb{R}$ is $\rho$-integrable, i.e. for any $(x,x') \in \mathcal{X}^2$,
\begin{align}
    \nonumber \int\int_{\mathcal{X}^2} k((x,x'),(\hat{x},\hat{x}'))d\rho(\hat{x})d\rho(\hat{x}')< \infty.
\end{align}
\end{assumption}

\section{Methodology}
In light of the core problem we aim to address, namely the computational complexity arising from the necessity to accommodate all preceding examples in the local error computation, our proposed methodology centers around mitigating this challenge. We begin by suggesting the utilization of a moving average of previous examples, thus alleviating the burden of handling the complete historical dataset.

However, it's imperative to note that the simple average approach may lead to misleading outcomes, especially when the data exhibits high variance distributions. Recognizing this potential limitation, we introduce a corrective measure to enhance the reliability of the average-based gradient. This correction mechanism involves introducing a random point from the historical loss, enhancing the robustness of the gradient estimation, and ameliorating the impact of potential average imbalances. This comprehensive strategy effectively tackles the issues associated with handling the entire historical dataset, while also accounting for data distribution complexities that can skew outcomes, and most importantly not requiring the iid assumption. Another parallel research addressing the high variance in the buffer has been proposed recently \cite{alquabeh2023variance}, however with $O(s)$ complexity.

Indeed, the loss function we seek to minimize can be formulated as a linear combination of two distinct losses: one based on the moving average and the other on a randomly chosen example. However, it's important to recognize that the minimization of this combined loss doesn't inherently guarantee the minimization of the true local loss. This discrepancy arises because, at the optimal solution where the gradient is zero, the gradients of both components remain nonzero, depending upon the linear coefficient constant. 

To navigate this challenge, an alternating strategy to minimize both losses is studied. In essence, we proceed by iteratively taking steps using the average-based gradient and then employing the random gradient on the transition model derived from the previous step. This alternating optimization strategy accounts for the intricate balance between the two gradient components and fosters a dynamic convergence process. Moreover, the iterative method ensures that both the average-based and random-based gradients are appropriately accounted for in the optimization process. The loss minimization upon receiving new examples $z_t$ can be formulated using the average $\bar{z} = (\sum_{i=1}^{t-1}\frac{x_i}{t-1},y_{t-1}|{y_{t-1}\neq y_t})$ and the random sample $\hat{z}$ from $\{z_1,\dots,z_{t-1}\}$, as follow;
\begin{equation}
   \min_w \ell \left(w - \eta_t \nabla \ell(w,z_t,\bar{z}),z_t,\hat{z} \right) 
\end{equation}
where $\eta_t$ is the step size. As detailed in the subsequent section, the magnitude of the step taken in the second update is contingent upon both the variance of the data and the smoothness of the loss function. Interestingly, when the variance is negligible, the step taken based on the average gradient is deemed adequate, rendering the random step unnecessary. The objective function can give more sense if we use second-order Taylor expansion after denoting $\hat{\ell}(\cdot) = \ell(\cdot,z_t,\hat{z})$ and $ \bar{\ell}(\cdot) = \ell(\cdot,z_t,\bar{z})$ as follow,
\begin{align}\nonumber
  \ell ( w - \eta_t \nabla \ell(w,z_t,\bar{z}),z_t,\hat{z} )  & \approx  \hat{\ell}(w) - \eta_t \nabla \hat{\ell}(w)^T \nabla \bar{\ell}(w) \\
   + \eta_t^2 & \nabla \bar{\ell}(w)^T  \nabla^2 \hat{\ell}(w)  \nabla \bar{\ell}(w)
\end{align}
Hence, the process of minimizing the objective function involves addressing three integral components. First, there is the endeavor to minimize the loss associated with randomly selected examples. Then, leveraging the convex nature of the loss function and the positive definiteness of its Hessian, the third term—strictly non-negative in nature—ensures that the gradient of the average-based loss is minimized. Of paramount significance is the role of the second term, which guarantees that the gradients of losses, both average-based and random-based, exhibit a maximum inner product. This intricate interplay is pivotal in correcting the moving average gradient, particularly in scenarios involving highly skewed distributions.

In order to handle the intricate patterns of real-world data, our approach to online pairwise learning involves a fundamental assumption: both the hypotheses and the data are transformed into a specialized space known as a Hilbert space, denoted as $\mathcal{H}$. This space is equipped with a specific type of kernel called a Mercer kernel. Essentially, this kernel is a mathematical function that captures the relationships between data points. This Hilbert space has two key attributes: first, it can reproduce the values of the kernel function, and second, it encompasses various combinations of these kernel functions along with their complete structures.

When we apply traditional buffer-based pairwise algorithms in an online setting, the process of kernelization comes with a complexity that grows as $O(T^2s)$, where $s$ represents the size of the buffer—a rather unwieldy task for larger problems. To address this challenge, we introduce a method that approximates the kernel function using a simpler approach. We assume the existence of a lower-dimensional mapping, represented as $r:\mathcal{X}^2\mapsto R^D$. This mapping enables us to simplify the kernel function's calculations, significantly reducing the complexity of the process. This simplified approach is achieved using what are known as random Fourier features, which are derived from the mathematical Fourier transform. Although these techniques have been studied in previous works, our focus is on adapting them to the online setting, and we establish that as few as $O(\sqrt{T}\log{T})$ random features are enough to achieve effective results.

Our method is introduced in algorithm \ref{alg:FPOGD} for general pairwise learning, and it includes the Gaussian kernel approximation using random Fourier features. In the upcoming section, we delve into the performance of this online algorithm and dissect its regret into two distinct components: the first is a direct result of learning a model from an online stream of examples, and the second arises from the approximation of the kernel function.

\begin{algorithm}[t]
	\caption{Average Online Gradient Descent (AOGD)}\label{alg:FPOGD}
	\begin{algorithmic}[1]
		\REQUIRE Random feature size $D$, initial solution $w_1\in \bar{\mathcal{H}}$, initial average and random example $\bar{z} = \hat{z} \in R^{D}$, kernel function $k$ and its Fourier transform $P(u)$, probability $p$, step sizes $\eta_t,\gamma_t$ and $\lambda$.
  		\STATE Sample Fourier feature $\{u_i\}_{i=1}^{D/2}\sim P(u_i)$
		\FOR {$t = 2,\dots,T$}
		\STATE Receive a new example $z_t=(x_t,y_t) \in \mathcal{Z}$
		\STATE Map to the new space:
		 \[z_t =( \frac{2}{\sqrt{D}}[\sin(u_i^Tx_t),\cos(u_i^Tx_t)]_{i=1}^{D/2},y_t)\]
		\STATE Suffer loss $\bar{\ell}(w_{t-1},z_t,\bar{z})$
		\STATE Update model $w'_t = w_{t-1} - \eta_t \nabla\bar{\ell}(w_{t-1},z_t,\bar{z})$
		\STATE Suffer loss $\hat{\ell}(w'_t,z_t,\hat{z})$
  		\STATE Update model $w_t = w'_t - \gamma_t \nabla\hat{\ell}(w'_t,z_t,\hat{z})$
		\STATE Update the moving average and the random example,
          \begin{align*}
             &  \bar{z} \leftarrow \frac{z * (t-1) + z_t}{t}\\
             &  \hat{z} \leftarrow z_t \text{ if Bernoulli(p) = 1}
             \end{align*}		\ENDFOR
             \RETURN $w_T$
	\end{algorithmic}
\end{algorithm}

\section{Regret Analysis}
The regret incurred by the online algorithm relative to the optimal hypothesis in the original space $\mathcal{H}$, denoted as $w^* = \arg\min_{w\in \mathcal{H}} \sum_{t=2}^T L_t(w)$, when applied to a sequence of $T$ examples, is formally expressed as:
\begin{align}\label{eq:reg}
   \mathcal{R}_{w^*,T} = \sum_{t=2}^T L_t(w_{t-1}) - \sum_{t=2}^T L_t(w^*),
\end{align}

In the context of this equation, the local all-pairs loss $L_t(\cdot)$ is defined according to Equation \ref{eq:localloss}.
To delve deeper, we can decompose the regret represented in Equation \ref{eq:reg} by introducing the best-in-class hypothesis within the approximated space $\bar{\mathcal{H}}$, denoted as $\bar{w}^* = \arg\min_{w\in \bar{\mathcal{H}}} \sum_{t=2}^T L_t(w)$. This decomposition takes the form:
\begin{align}\label{eq:regret}\nonumber
\mathcal{R}_{w^*,T} =& \underbrace{\sum_{t=2}^T L_t(w_{t-1}) - \sum_{t=2}^T L_t(\bar{w}^*)}_{T_1}\\
&+ \underbrace{\sum_{t=2}^T L_t(\bar{w}^*)  - \sum_{t=2}^T L_t(w^*)}_{T_2}
\end{align}

In our analysis, we provide a bound for $T_1$ in Theorem \ref{th:Allpairs}, followed by an error bound for $T_2$ in Theorem \ref{regret_2}. By combining these bounds, we derive the regret bound for Algorithm \ref{alg:FPOGD} as follows:

\begin{theorem}
Let $\{z_t\in \mathcal{Z}\}_{t=1}^T$ be sequentially accessed by Algorithm \ref{alg:FPOGD}. Let $D$ denote the number of random Fourier features in the kernel mapping from the original space $\mathcal{X} \subset \mathbb{R}^d$. Let $\eta$ be first step size, $\gamma=O(\Gamma M_t \eta)$ the second step size, and $G$ the Lipschitz constant. Then the regret bound compared to $w^* = \arg\min_{w\in \mathcal{H}} \sum_{t=2}^T L_t(w)$ is bounded with a probability of at least $1-2^8\left(\frac{\sigma Diam(\mathcal{X})}{\epsilon}\right)\exp(-D\epsilon^2/(4d+8))$, as follows:
\begin{align}
\mathcal{R}_{w^,T} &\leq G T \|w^*\|_1 \epsilon  + \frac{ \| \bar{w}^*\|^2}{2\eta} + 2 \gamma G^2T 
\end{align}
where, $\epsilon$ denotes the kernel approximation error, $\sigma$ signifies the kernel width, $\Gamma \geq tr(cov[z_t]) $, and $M_t:M(w_t)$ stands for the smoothness parameter of $\ell$ in the neighborhood of $w_t$.
\end{theorem}

\begin{remark}
Selecting $\eta=\frac{1}{\sqrt{T}}$ and $\epsilon=\frac{1}{\sqrt{T}}$ results in a sublinear bound, which is optimal. Notably, $D = O(T)$ in general, yet it can be as low as $\sqrt{T}\log{T}$ for specific kernels. Furthermore, since $\gamma=O(\Gamma M_t \eta)$, the regret depends on the variance of data and the smoothness of the loss.
\end{remark}
\subsection{Regret in the Approximated Space $\bar{\mathcal{H}}$}

Prior approaches (e.g., \citep{kar13}, \citep{zhao2011online}) aimed to attain uniform convergence of buffer loss often relied on the assumption of iid examples. These methods also require maintaining uniform samples from historical data through a Bernoulli process with diminishing probabilities, ultimately approaching zero. Moreover, achieving such uniform convergence necessitated a disentanglement between the model and the buffer to ensure the validity of expectation. For instance, while \cite{zhao2011online} employed Reservoir sampling to enforce an $iid$ buffer, they overlooked the interplay between the model and the data. Similarly, \citep{kar13} utilized modified reservoir sampling and Symmetrization of Expectations to address these challenges. However, their approach demanded a buffer size of $O(\sqrt{T})$, which proves impractical for larger-scale scenarios. In contrast, \cite{yang2021simple} accomplished sublinear regret with a buffer size of one, focusing on local loss. Nevertheless, their assumption of a fully $iid$ training set does not align with real-world online situations. Our forthcoming theorem obviates the necessity for the $iid$ assumption by harnessing both moving averages and individual random points.

As previously explained, in each iteration, a single moving average is employed to ascertain the gradient direction, which establishes a connection to the all-pairs loss without relying on the $iid$ assumption. This is then followed by a step involving the gradient of a random example. Notably, this sample should not be uniformly chosen from historical data. If we denote the moving average based loss as $\bar{L}_t(w_{t-1})$, then using the convexity of $\ell(\cdot)$ and the linearity of models in $\bar{\mathcal{H}}$, we have $\bar{L}_t(w_{t-1})\leq  L(w_{t-1})$ or $\bar{L}_t(w_{t-1}) =  L(w_{t-1}) - g(t)$ where $g(t)\geq 0$ is Jensen's gap.
 The gap is first bounded in lemma \ref{lemma:Jensen} using the smoothness constant and the variance of the data.

\begin{lemma}\label{lemma:Jensen}
    Given that the data $\{z_1,\dots,z_t\}$ is linearly separable and have bounded total variance, i.e. $\Gamma \geq tr(cov[z_t]) $ and a loss function with $M_t:M(w_t)$ smoothness parameters, the Jensen's gap can be bounded as follow,
      \begin{align}
      g(t)  \leq \frac{1}{2}  \Gamma M_t    \end{align}
\end{lemma}
 Note that since the model is linear in the space $\bar{\mathcal{H}}$, the smoothness assumption holds for loss gradient w.r.t. $z$. Therefore we can provide the following bound in theorem \ref{th:Allpairs}.
 
\begin{theorem}\label{th:Allpairs}
Assume that assumptions \ref{ass:Lipschitz continuous},\ref{ass:smmoth}, and \ref{ass:Convexity} hold. Let $[w_t]_{i=1}^T$ be the sequence of models returned by running the algorithm \ref{alg:FPOGD} for $T$ times using the online sequence of data. Then if $\bar{w}^* =\arg\min_{w \in \bar{\mathcal{H}}} \sum_{t=2}^T L_t(w)$, we have,
   \begin{align}\label{allpairsregret}
  \sum_{t=2}^T L_t(w_{t-1})  &\leq  \sum_{t=2}^T L_t(\bar{w}^*)
  + \frac{ \| \bar{w}^*\|^2}{2\eta} + 2 \gamma G^2T  
\end{align} 
Note that $\Gamma \geq tr(cov[z_t])$ where $cov[z_t]$ is the covariance of the received examples, and $M_t$ is the smoothness parameter of the loss function.
\end{theorem}

\subsection{Approximation Error for Random  Features}\label{sec4.3}
The kernel associated with the space $\mathcal{H}$ is a function defined as  $k:\mathcal{X}^2\times \mathcal{X}^2 \mapsto \mathbb{R}^+$ with a shorthand $k_{(x,x')}(\cdot):=k((x,x'),(\cdot,\cdot))$  and can be constructed  given any uni-variate kernel $\mathcal{G}$ for any $x_1,x_2,x_1',x_2'\in\mathcal{X}$ as follow,
\begin{align}\label{eq:pairwiseKernel1}
 &   k(x_1,x_2,x_1',x_2') = \mathcal{G}(x_1,x_1') + \mathcal{G}(x_2,x_2') \\\nonumber
 &    - \mathcal{G}(x_1,x_2') - \mathcal{G}(x_2,x_1') = \langle \mathcal{G}_{x_1} - \mathcal{G}_{x_2}, \mathcal{G}_{x_1'} - \mathcal{G}_{x_2'} \rangle_\mathcal{G}
\end{align}
The pairwise kernel $k$ defined earlier demonstrates positive semi-definiteness over $\mathcal{X}^2$. This property aligns it with the attributes of a Mercer kernel, as indicated in references like \cite{ying2015online}, when $\mathcal{G}$ exhibits such characteristics over $\mathcal{X}$.
We further assume there exist a lower dimensional mapping $r:\mathcal{X}\mapsto R^D$, such that, $ \mathcal{G}_{x}(\cdot) \approx r(x)^Tr(\cdot)$

The literature extensively explores the accuracy of approximating the pointwise kernel $\mathcal{G}$ using random Fourier features, as discussed in references such as \cite{rahimi2007random}, \cite{bach2017equivalence}, and \cite{li2022sharp}.  
For example if the kernel is the Gaussian kernel, i.e., $\mathcal{G}(x-x') = \frac{1}{\sqrt{2 \pi}d} \exp{-(x-x')^T[diag(\sigma)](x-x')}$, we have $p(u)$ from Fourier transform to be $
    p(u) \propto \mathcal{N}(\mathbf{0},diag(\sigma))
$
where $\sigma \in \mathbb{R}^d$ is the kernel width. Thus kernel $\mathcal{G}$ can be given by a Monte Carlo estimate as follow,
\begin{align}
    \hat{\mathcal{G}}(x,x') \approx \frac{2}{D}\sum_{i=1}^{D/2} \cos(u_i^T(x-x')) = \langle r(x),r(x')\rangle
\end{align}
Where $u_i \sim \mathcal{N}({0},diag(\sigma))$, and  $r(x) := \frac{1}{\sqrt{D/2}}[cos(u_i^Tx),sin(u_i^T,x)]_{i=1}^{D/2}$. The quality of the estimation is given by \cite{rahimi2007random} to be dependent on the number of random features $D$ used to approximate the Fourier transform, which can be extended to pairwise kernels as illustrated in the following corroallary. 

\begin{corollary}\label{RFerror}
Given $x_1,x_2,x_1',x_2'
\in\mathcal{X}$, and pairwise kernel $k$ defined on $\mathcal{X}^2\times\mathcal{X}^2$, the random Fourier estimation of the kernel has mistake bounded with probability at least $1-8\exp{\frac{-D^2\epsilon}{2}}$ as follow,
\begin{align}\nonumber
    \  |\hat{k}_{(x_1,x_2)}(x_1',x_2')   - {k}_{(x_1,x_2)}(x_1',x_2')  | \leq \epsilon  \end{align}
\end{corollary}
The following theorem bounds the random Fourier error in equation (\ref{eq:regret}) using the results of corollary \ref{RFerror}.
\begin{theorem}
\label{regret_2}
Given a pairwise Mercer kernel $k_{(x,x')} :=k((x,x'),(\cdot,\cdot))$ defined on $\mathcal{X}^2\times\mathcal{X}^2$. Let $\ell(w,z,z')$ be a convex loss that is Lipschits continuous with constant ${G}$. Then for any $w^* = \sum_{i,j\neq i}^Ta_{i,j}^* k_{(x_i,x_j)}$, and random Fourier features number $D$ we have the following with probability at least $1-2^8\left(\frac{\sigma Diam(\mathcal{X})}{\epsilon}\right)\exp(-D\epsilon^2/(4d+8))$,
\begin{align}
    \sum_{t=2}^T L_t(\bar{w}^*)  - \sum_{t=2}^T L_t(w^*) \leq 
 G T \|w^*\|_1 \epsilon
\end{align}
where $ \|w^*\|_1 = \sum_{i,j\neq i}^T|a^*_{i,j}|$.
\end{theorem}
 \begin{remark}
A trade-off between $\|w^*\|_1$ and $\epsilon$ exists when independent support vectors are in the kernel space ($|a^*_{i,j} |> 0$). In such cases, a larger error requires a smaller $\epsilon$, implying a larger $D$. Regularization can set insignificant support vectors' $a^*_{i,j}$ to zero, a subject for potential future research.
 \end{remark}

\section{Related Work}
While online gradient descent exhibits a time complexity of $O(T^2)$ \citep{boissier2016fast,kar13,gao2013one}, this remains impractical for large-scale problems. For example, in AUC maximization, pairing a data point $x_t, y_t=1$ with all preceding ${x_{t'}, y_{t'}=-1|1<t'<t-1}$ to ascertain the true loss poses a key challenge. This approach necessitates computing gradients $\nabla {L}t(f{t-1})$ for all received training examples, growing linearly with $t$.

Buffer-based methods have been explored in the literature to manage this challenge. \citep{zhao2011online} introduced buffers $B_+$ and $B_-$, each of size $N_+$ and $N_-$, considering only points within the buffers at each step $t$. Reservoir Sampling updates the buffers while ensuring uniform sampling. Theorem 1 in \citep{zhao2011online} establishes sublinear regret linked to buffer sizes. However, this work is limited to AUC with linear models and does not explicitly address the buffer size's impact on generalization error.

In contrast, \citep{kar13} proposed modified reservoir sampling (RSx) by performing $s$ Bernoulli processes with probability $1/t$ to replace each buffer point with the received data point. This approach ensures $s$ data points are iid samples from the stream and independent from $w$, with the generalization bound depending on Rademacher complexity and regret. Nevertheless, optimal generalization requires $s=O(\sqrt{T})$ for buffer loss.

Another parallel research addressing the high variance in the buffer has been proposed recently \cite{alquabeh2023variance}, coinciding with the publication of our paper. Recently, \citep{yang2021simple} suggested optimal generalization with a buffer size of $s=1$, assuming $iid$ data and independent buffer-data relations. Their work aligns with ours, but we remove the $iid$ assumption for convergence proof and consider non-linear model spaces. 

\section{Experiments}
We perform experiments on several real-world datasets and compare our proposed algorithm to both offline and online pairwise algorithms. Specifically, the proposed method is compared with different pairwise algorithms on the AUC maximization task with the squared function as the surrogate loss.

\subsection{Experimental Setup}

\begin{table*}[!t]
\label{table:AUC}
   	\small
 \center
\begin{tabular}{lllllll}
\hline
{\textbf{Dataset}} & {\textbf{AOGD}} & {\textbf{SPAM-NET}} & {\textbf{OGD}} & \textbf{Sparse Kernel} & \textbf{Projection++} & \textbf{Kar}\\ 
\hline 
diabetes  & {81.91}$\pm$0.48        & 82.03$\pm$0.32          & 82.53$\pm$0.31            & 82.64$\pm$0.37           & 77.92$\pm$1.44   & {79.85}$\pm$0.28      \\

ijcnn1                   & \textbf{92.32}$\pm$0.77                 & 87.01$\pm$0.10          & 89.46$\pm$1.22            & 71.13$\pm$0.59           & 92.20$\pm$0.27 & 83.44$\pm$1.21           \\

 a9a  & \textbf{90.03}$\pm$0.41  & 89.95$\pm$0.42 & 90.01$\pm$0.42           & 84.20$\pm$0.17 & 84.42$\pm$0.33  & 77.93$\pm$1.55 \\

mnist &\textbf{96.98}$\pm$0.38 &{96.57}$\pm$0.54 & 89.56$\pm$0.34& 95.21$\pm$0.15&93.82$\pm$0.15 &89.16$\pm$0.14\\

 rcv1   &\textbf{ 99.38}$\pm$0.48   & 98.13$\pm$0.15 &  99.05$\pm$0.57  & 96.26$\pm$0.35 & 94.54 $\pm$0.36 & 97.78 $\pm$0.64\\
  usps   & \textbf{92.76}$\pm$0.48   & 89.12$\pm$0.88 &  89.88$\pm$0.47  & 91.25$\pm$0.84 &90.14 $\pm$0.22 &  89.58$\pm$0.25 \\
    german   & \textbf{80.75}$\pm$0.48   & 78.55$\pm$2.75 &  79.25$\pm$0.42  & 80.11$\pm$0.44 &78.44 $\pm$0.66 &  77.24$\pm$0.52\\
\hline
Reg. & $l_2$  &    $l_1+l_2$  &    $l_2$ & $l_2$  &$l_2$& $l_2$ \\
\hline 
\end{tabular}
\caption{AUC maximization results (average  $\pm$ standard error)$\times10^2$ using different batch and online algorithms on different datasets}
\end{table*}

\begin{figure*}        
\label{fig:results}
     \centering
         \includegraphics[width=0.32\textwidth]{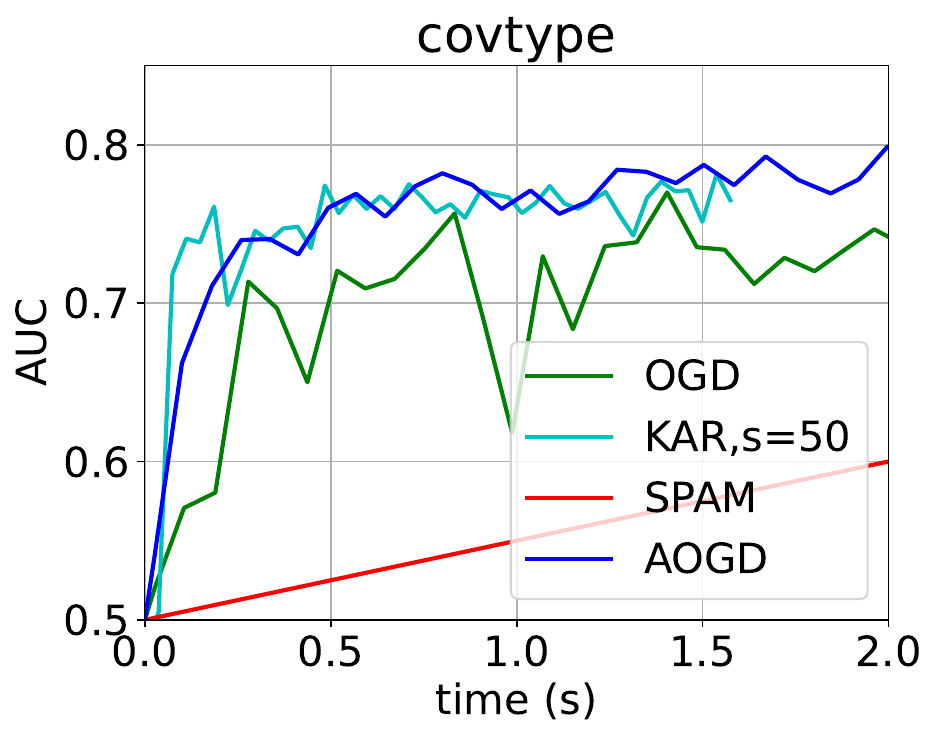}
         \includegraphics[width=0.32\textwidth]{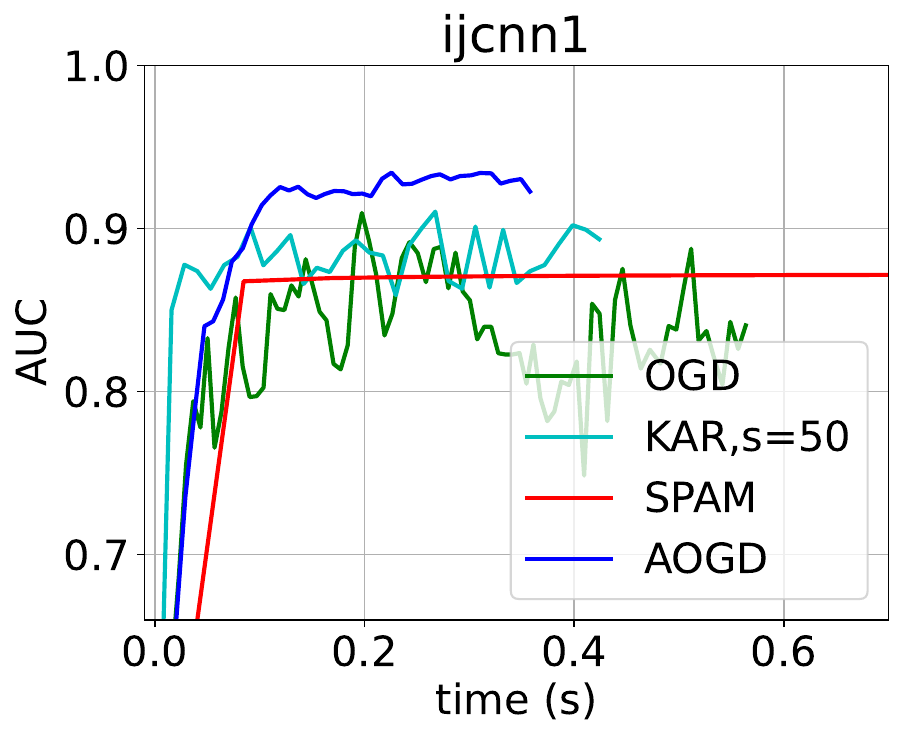}
         \includegraphics[width=0.32\textwidth]{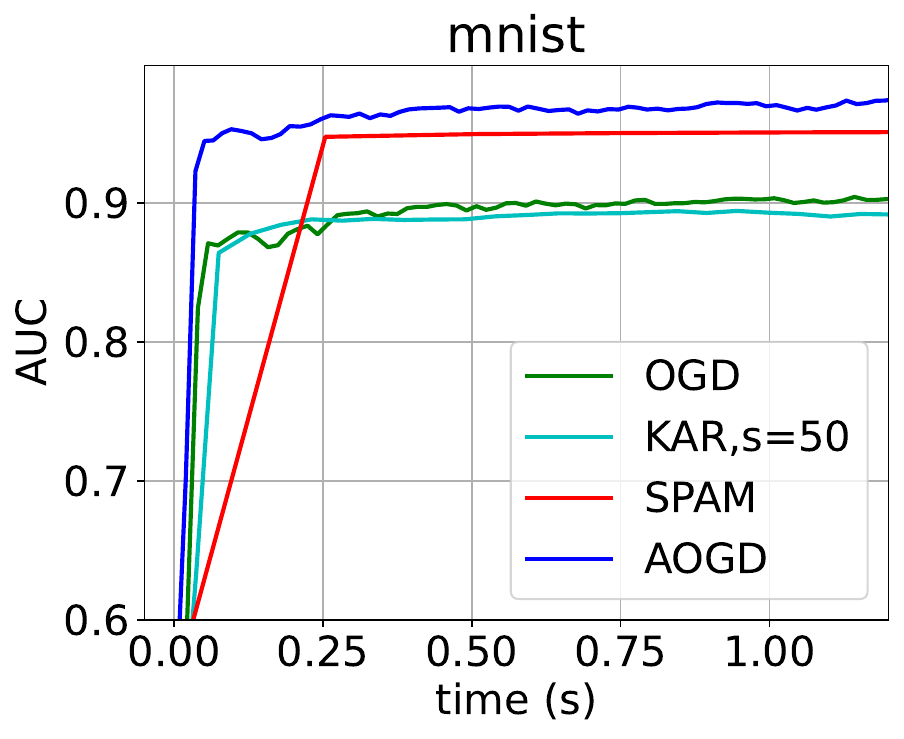}
         \includegraphics[width=0.32\textwidth]{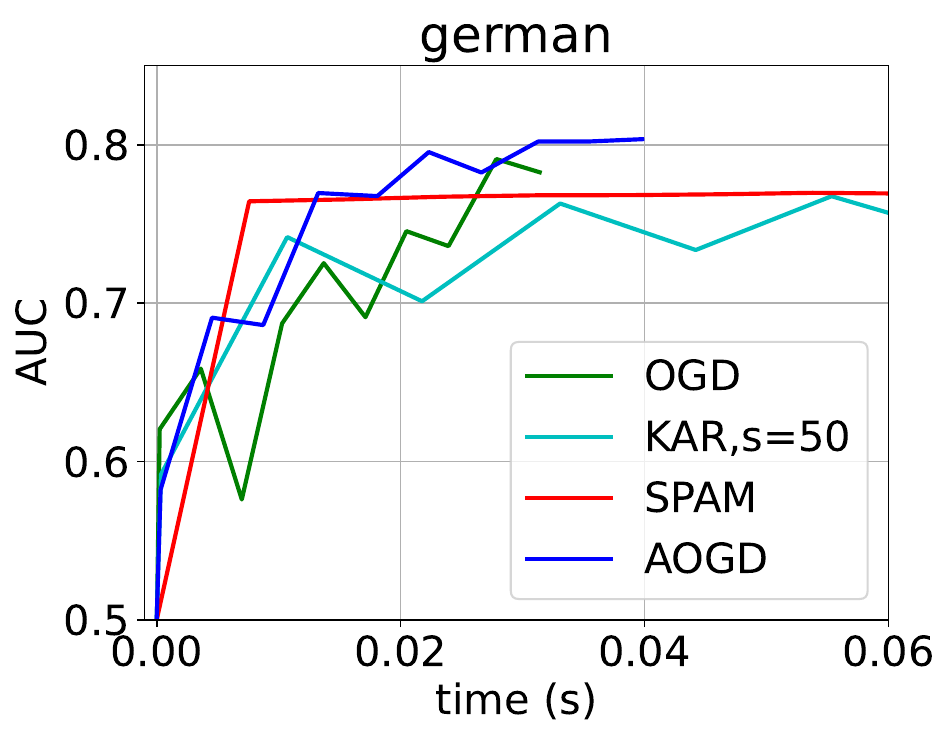}
        \includegraphics[width=0.32\textwidth]{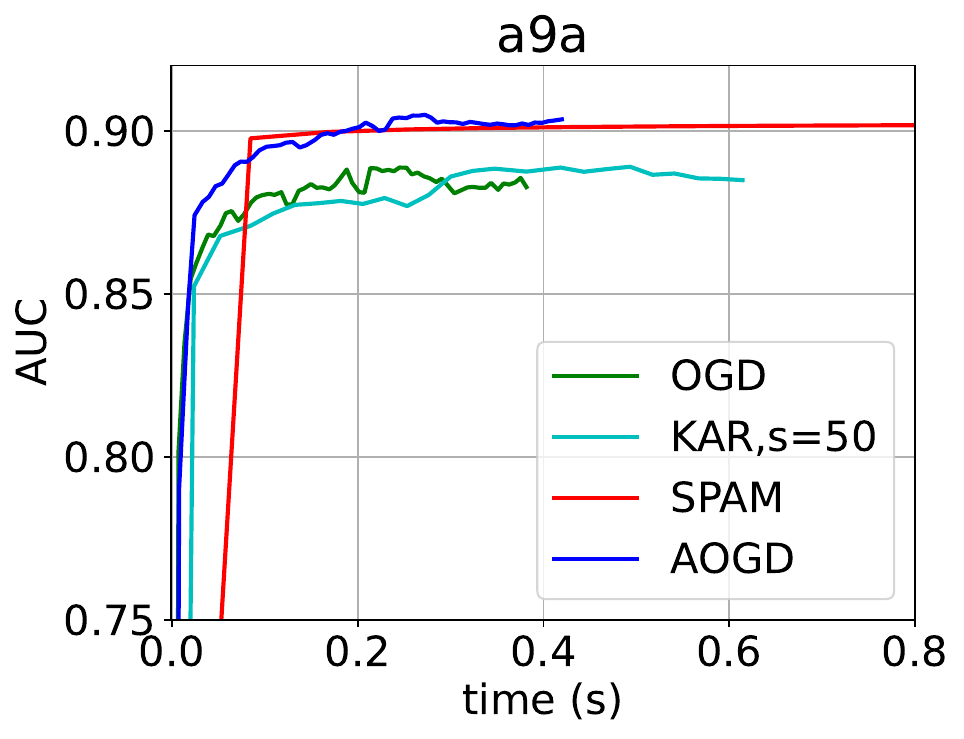}
         \includegraphics[width=0.32\textwidth]{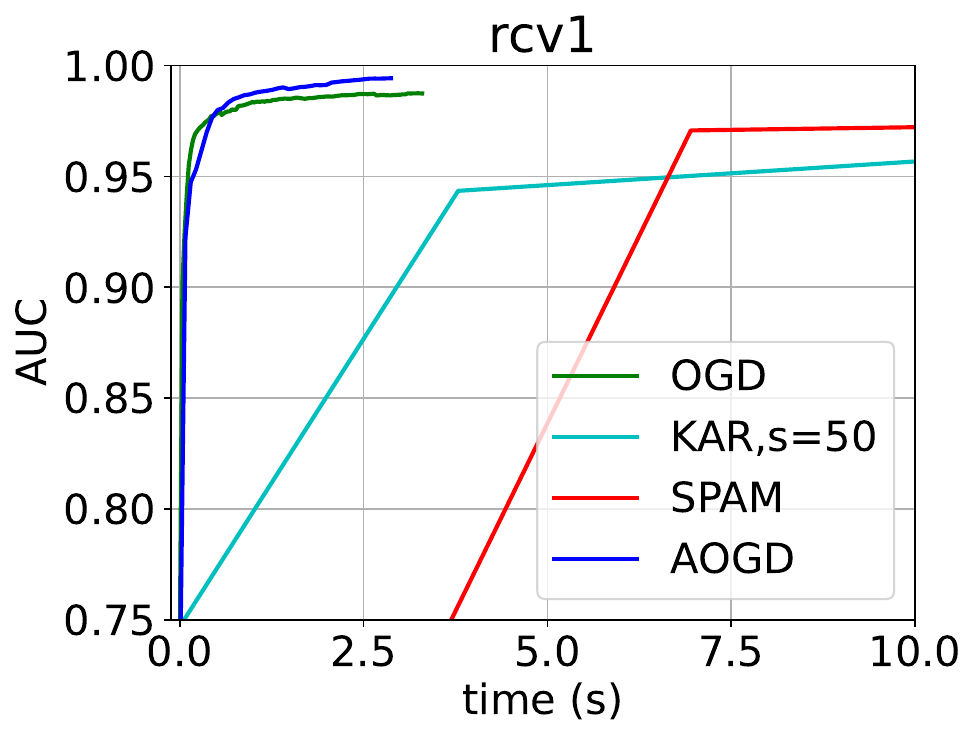}

             \caption{The AUC vs. time comparison of the algorithms in different datasets showing superior performance of the proposed method. }
             \centering
\end{figure*}
\paragraph{Compared Algorithms.} The compared algorithms including offline and online settings, as illustrated below,
\begin{itemize}
    \item SPAM-NET \citep{reddi2016stochastic} is an online algorithm for AUC with square loss that is converted to a saddle point problem with non-smooth regularization.
    \item OGD \citep{yang2021simple}, the most similar to our algorithm but with a linear model, that uses the last point every iteration.
    \item Sparse Kernel \citep{kakkar2017sparse} is an offline algorithm for AUC maximization that uses the kernel trick.
    \item Projection ++ \citep{hu2015kernelized} is an online algorithm with an adaptive support vector set.
    \item Kar \citep{kar13} is online algorithm with buffer of size $s$.
\end{itemize}



\paragraph{Implementation.} The datasets of investigation are available on LIBSVM website \citep{CC01a}. The experiments are validated for all algorithms by performing a grid search on the hyperparameters with three-fold cross-validation. For instance, in all algorithms  the step size $\eta\in 2^{[-8:-1]}$, and the regularization parameters $\lambda \in 10^{[-8:-1]}$.  All the algorithms have been run five times on different folds on a CPU of 4 GHz speed and 16 GB memory.
\subsection{Experimental Results and Analysis}
Our results on maximizing the area under the curve (AUC) using a squared loss function verify the efficacy of our random Fourier pairwise online gradient descent procedure.
Table 2 shows that when compared to online and offline, linear and nonlinear pairwise learning algorithms, our algorithm achieves improved AUC performance on large-scale datasets.
\par
Figure 1 displays the AUC performance relative to CPU utilization.
Results of ijcnn1 and mnist show that compared to other state-of-the-art online pairwise learning methods, our algorithm not only has better AUC performance but also has a faster convergence rate.
As compared to the two most popular kernel algorithms, Sparse kernel, and Projection++, our technique provides superior AUC performance across the board.
Projection++ is impractical for large-scale datasets because of its hyper-parameters and the requirement to regularly update a support vector set.

\section{Conclusion}
In summary, our work introduces a lightweight online kernelized pairwise learning algorithm applicable to both linear and non-linear models. The algorithm efficiently computes the online gradient using a moving average and random examples, with the kernel function approximated through random Fourier mapping. Achieving a gradient complexity of $O(T)$ for linear models and $O(\frac{D}{d}T)$ for nonlinear models, our approach showcases computational efficiency. The incorporation of multiple moving averages, a potential avenue for future research, could further enhance adaptability and robustness.
Notably, we address the challenges of non-iid data by dual evaluation of gradients, leveraging random Fourier features for efficient kernel estimation with a sublinear error bound. This integration of kernelization, approximation, and dynamic averaging transcends linear constraints, effectively tackling non-iid data challenges and mitigating kernel complexity.

\bibliography{refs}

\onecolumn

\section{Appendix}

\section{Proof of Theorem 2}
Given the linearity of the model, expressed as $f_w(x) = \langle w, x \rangle$, the convexity of the loss function in the weight variable $w$ implies convexity in the input variable. In other words, the linearity of the model suggests that if the loss function is convex with respect to the weight variable $w$, it will also exhibit convexity concerning the input variable. This connection further reinforces the relationship between the convexity of the loss function, linearity of the model, and convexity properties in both the weight and input variables. Therefore, given the convexity property of the loss function $\ell$, upon the arrival of a new example $z_t$, the loss based on the average, denoted as $\bar{z} := (\mathbb{E}[x | y \neq y_t], y) = \left(\sum_{i=1}^{t-1}\frac{x_i}{t-1}, y_{t-1} | {y_{t-1} \neq y_t}\right)$, can be linked to the local all-pairs loss using Jensen's inequality as follow,

\begin{equation}
\ell(w, z_t, \bar{z}) \leq \frac{1}{t-1}\sum_{t=1}^{t-1} \ell(w, z_t, z_i)
\end{equation}

Alternatively, representing the average-based loss as $\bar{L}(w)$ and the cumulative loss as $L(w) = \frac{1}{t-1}\sum_{t=1}^{t-1} \ell(w, z_t, z_i)$, we introduce a slack variable at $w = w_{t-1}$ as follows,

\begin{equation}
\bar{L}_t(w{t-1}) = L(w_{t-1}) - g(t)
\end{equation}

Here, $g(t) \geq 0$ denotes Jensen's gap. This relation establishes a mathematical connection between the average-based loss and the all-pairs loss through the introduction of this slack variable. Additionally, it's noteworthy that this gap can be bounded by leveraging the smoothness of the loss function and considering the variance inherent in the received variables, as demonstrated in Lemma \ref{lemma:Jensen1}. 
\begin{lemma}\label{lemma:Jensen1}
    Given that the data $\{z_1,\dots,z_t\}$ is linearly separable and have bounded total variance, i.e. $\Gamma \geq tr(cov[z_t]) $ and a loss function with $M_t:M(w_t)$ smoothness parameters at $w=w_t$, the Jensen's gap can be bounded as follow,
      \begin{align}
      g(t)  \leq \frac{1}{2}  \Gamma M_t 
    \end{align}
    
\end{lemma}
\begin{proof}
   Given that the loss function $\ell(w_{t-1})$ is $M_t$-smooth, and considering the linearity of the model in the approximated space $\bar{\mathcal{H}}$, expressed as $f_w(x) = \langle w, x \rangle$, it follows that if the function is smooth in $w$, it is also smooth in $x$ (or $z$ in our context). Let $Q := {z_1, \dots, z_{t-1}}$ represent the set of historical examples. Define $\bar{z} := (\mathbb{E}[x | y \neq y_t], y)$ where the expectation is taken with respect to the uniform probability over the set $Q$. Utilizing the smoothness property of $\ell$ from assumption \ref{ass:smmoth} and considering expectations,
    \begin{align}\nonumber
        \ell(w_{t-1},z_t,z) & \leq \ell(w_{t-1},z_t,\bar{z}) + \nabla_{z} \ell(w_{t-1},z_t,\bar{z})^T(z-\bar{z}) + \frac{M_t}{2}(z-\bar{z})^T(z - \bar{z}) \\ \nonumber
        \mathbb{E}_z \ell(w_{t-1},z_t,z) -\ell(w_{t-1},z_t,\bar{z})   & \leq  \mathbb{E}_z \nabla_{z} \ell(w_{t-1},z_t,\bar{z})^T(z-\bar{z}) +\mathbb{E}_z \frac{M_t}{2}(z-\bar{z})^T(z - \bar{z}) \\ \nonumber
        L(w_{t-1}) - \bar{L}(w_{t-1})  & \leq \mathbb{E}_z \nabla_{z} \ell(w_{t-1},z_t,\bar{z})^T(z-\bar{z}) + \mathbb{E}_z\frac{M_t}{2}(z-\bar{z})^T(z - \bar{z})    \\ \nonumber
        & = \frac{M_t}{2} \mathbb{E}_z \|z - \bar{z}\|^2=  \frac{M_t}{2} \mathbb{E}_z \sum_i^d (z_i - \bar{z}_i)^2 \\ 
        &=  \frac{M_t}{2} \sum_i^d \mathbb{E}_z(z_i - \bar{z}_i)^2 =  \frac{M_t}{2} \sum_i^d Var[z_i] =  \frac{M_t}{2} tr(cov[z])
    \end{align}
    Replace the expression for $g(t)$ in order to derive the outcomes. 
\end{proof}
It's important to highlight that $tr(\text{cov}[z])$ denotes the trace of the covariance matrix related to the variable $z$ (or the feature $x$ of $z$). We will colloquially refer to this as the total variance.
Next we present the following lemma, which provides an upper bound on the buffer loss (based on the average), before we establish the regret bounds for the all-pairs loss in Theorem 2.
\begin{lemma}\label{theorem:regret 1}
Assume that assumptions \ref{ass:Lipschitz continuous} and \ref{ass:Convexity} hold. Let $[w_t]_{i=1}^T$ be the sequence of functions returned by running the algorithm 1 for $T$ times using online sequence of data. Then if $\bar{w}^* =\arg\min_{w \in \bar{\mathcal{H}}} \sum_{t=2}^T {L}_t(w)$, we have the following: 
\begin{align}
    \sum_{t=2}^T \bar{L}_t(w_{t-1}) \leq   \sum_{t=2}^T \bar{L}_t(\bar{w}^*) +\frac{\|\bar{w}^*\|^2 }{2\eta}+  \gamma G^2T  -\sum_{t=2}^{T} \frac{\gamma}{\eta}\hat{v}_t^T(w_{t-1}-\bar{w}^*)
\end{align}
\end{lemma}

\begin{proof}
Let $\bar{\ell}_t(\cdot)=\ell(\cdot,z_t,\bar{z})$ be the loss based on the average of history examples, where $\bar{z} := (\mathbb{E}x|{y\neq y_t,y})$ and the expectation is w.r.t. uniform probability over previous examples  , and $\bar{v}_t \in \partial \bar{\ell}_t(w_{t-1})$ then the first update is $ w'_t:= w_{t-1} - \eta_t \bar{v}$. Let $\hat{\ell}_t(\cdot) = \ell(\cdot,z_t,z) $ where $z$ is a random sample from $\{z_1,\dots,z_{t-1}\}$ . Let $\hat{v}_t \in \partial \hat{\ell}_t(w'_t)$. If we take the distance of two subsequent models to the optimal model we have, 
 	\begin{align}
 	\nonumber&   \| w_t - \bar{w}^* \|^2 - \| w_{t-1} - \bar{w}^* \|^2= \| w_{t-1} - \eta_t (\bar{v}_t +\frac{\gamma_t}{\eta_t}\hat{v}_t) -  \bar{w}^* \|^2 - \| w_{t-1} - \bar{w}^* \|^2 \\\nonumber
 	&=\| w_{t-1} - \bar{w}^* \|^2-2 \eta_t (\bar{v}_t +\hat{v}_t)^T(w_{t-1}-\bar{w}^*) + \eta_t^2 \|\bar{v}_t +\hat{v}_t\|^2-\| w_{t-1}- \bar{w}^* \|^2\\ \nonumber
 	& = -2 \eta_t (\bar{v}_t +\frac{\gamma_t}{\eta_t}\hat{v}_t)^T(w_{t-1}-\bar{w}^*) + \eta_t^2 \|\bar{v}_t +\frac{\gamma_t}{\eta_t}\hat{v}_t\|^2 \\ \nonumber
 	& \overset{(a)}{\leq} -2 \gamma_t \hat{v}_t^T(w_{t-1}-\bar{w}^*) -2 \eta_t ( \bar{\ell}_t(w_{t-1}) - \bar{\ell}_t(\bar{w}^*) ) +  \eta_t^2 \|\bar{v}_t +\frac{\gamma_t}{\eta_t}\hat{v}_t\|^2 \\ \nonumber
 	& \overset{(b)}{\leq} -2 \gamma_t \hat{v}_t^T(w_{t-1}-\bar{w}^*) -2 \eta_t ( \bar{\ell}_t(w_{t-1}) - \bar{\ell}_t(\bar{w}^*) )  +  2\gamma_t\eta_t G^2
 	\end{align}
 Where the inequality (a) implements assumption \ref{ass:Convexity}, i.e. $ -v_t^T(w_{t-1}-\bar{w}^*) \leq -(g_t(w_{t-1}) - g_t(\bar{w}^*))$. And inequality (b) implements assumption 1, i.e. $\|\nabla \ell(\cdot)\|\leq G$.
Setting the step size $\eta_t=\eta$ and $\gamma_t=\gamma$ for all $t$, we can get
	\begin{align}\nonumber
	&\bar{\ell}(w_{t-1},z_t,\bar{z}) - \bar{\ell}(\bar{w}^*,z_t,\bar{z})\\ 
    &\quad\leq \frac{  \| w_{t-1}-\bar{w}^*\|^2 - \|w_t - \bar{w}^*\|^2}{2\eta}+\gamma G^2  -\frac{\gamma}{\eta} \hat{v}_t^T(w_{t-1}-\bar{w}^*) 
	\end{align}

Finally if we sum from $t=2$ to $t=T$ and set $w_1=0$.
\begin{align}\nonumber
\sum_{t=2}^{T}\ell(w_{t-1},z_t,\bar{z})  - \sum_{t=2}^{T}  \ell(\bar{w}^*,z_t,\bar{z}) & \leq \frac{\|\bar{w}^*\|^2 - \|w_T - \bar{w}^*\|^2}{2\eta}+  \sum_{t=2}^{T} \gamma G^2 -\sum_{t=2}^{T} \frac{\gamma}{\eta}\hat{v}_t^T(w_{t-1}-\bar{w}^*) \\
& \leq \frac{\|\bar{w}^*\|^2 }{2\eta}+  \gamma G^2T  -\sum_{t=2}^{T} \frac{\gamma}{\eta}\hat{v}_t^T(w_{t-1}-\bar{w}^*) 
\end{align}
Substitute the definition of $\bar{L}_t(\cdot) =\ell(\cdot,z_t,\bar{z})$ to have the results. This completes the proof. 
\end{proof} 
Now we are ready to prove theorem 2 using the previous two lemmas.
\begin{theorem}[Theorem 2 restated]\label{th:Allpairs1}
Assume that assumptions \ref{ass:Lipschitz continuous},\ref{ass:smmoth}, and \ref{ass:Convexity} hold. Let $[w_t]_{i=1}^T$ be the sequence of models returned by running the algorithm for $T$ times using online sequence of data. Then if $\bar{w}^* =\arg\min_{w \in \bar{\mathcal{H}}} \sum_{t=2}^T {L}_t(w)$, we have the following,
   \begin{align}\label{allpairsregret1}
  \sum_{t=2}^T L_t(w_{t-1})  &\leq  \sum_{t=2}^T L_t(\bar{w}^*)
  + \frac{ \| \bar{w}^*\|^2}{2\eta} + 2 \gamma G^2T  
\end{align} 

\end{theorem}
\begin{proof}
Starting from lemma 2 final's inequality,
    \begin{align}\nonumber
    \sum_{t=2}^T \bar{L}_t(w_{t-1})  &\underset{}{\leq}   \sum_{t=2}^T \bar{L}_t(\bar{w}^*)   + \frac{\|\bar{w}^*\|^2 }{2\eta}+  \gamma G^2T  -\sum_{t=2}^{T} \frac{\gamma}{\eta}\hat{v}_t^T(w_{t-1}-\bar{w}^*) \\\nonumber
 &\overset{(a)}{\leq}  \sum_{t=2}^T {L}_t(\bar{w}^*) + \frac{\|\bar{w}^*\|^2 }{2\eta}+  \gamma G^2T  -\sum_{t=2}^{T} \frac{\gamma}{\eta}\hat{v}_t^T(w_{t-1}-\bar{w}^*) \\\nonumber
  \sum_{t=2}^T {L}_t(w_{t-1}) -\sum_{t=2}^T {L}_t(\bar{w}^*) &\overset{(b)}{\leq} \frac{\|\bar{w}^*\|^2 }{2\eta}+  \gamma G^2T  + \sum_{t=2}^T\frac{ M_t }{2}\Gamma- \frac{\gamma}{\eta}\nabla \ell(w_t',z_t,\hat{z})_t^T(w_{t-1}-\bar{w}^*) \\\nonumber
   \overset{(c)}{=} \frac{ \| \bar{w}^*\|^2}{2\eta} &+ \gamma G^2T  + \sum_{t=2}^T\frac{ M_t }{2}\Gamma- \frac{\gamma}{\eta}\nabla \ell(w_t',z_t,\hat{z})^T(w'_{t}-\bar{w}^* + \eta \nabla \ell(w_{t-1},z_t,\bar{z}))
   \\\nonumber
     \overset{(d)}{\leq} \frac{ \| \bar{w}^*\|^2}{2\eta} &+\gamma G^2T  + \sum_{t=2}^T\frac{ M_t }{2}\Gamma- \frac{\gamma}{2\eta M}\|\nabla \ell(w_t',z_t,\hat{z})\|^2 \\\nonumber
     &-\gamma \nabla \ell(w_t',z_t,\hat{z})^T\nabla \ell(w_{t-1},z_t,\bar{z})\\
      &\overset{(e)}{\leq}  \frac{ \| \bar{w}^*\|^2}{2\eta} +2\gamma G^2T  + \sum_{t=2}^T\frac{ M_t }{2}\Gamma- \frac{\gamma}{2\eta M}\|\nabla \ell(w_t',z_t,\hat{z})\|^2 
\end{align}
where inequality (a) implements the convexity of the loss, inequality (b) implements lemma \ref{lemma:Jensen1}, equality (c) implements update $w'_t +\eta \nabla \ell(w_{t-1},z_t,\bar{z})= w_{t-1}$, inequality (d) implements the smoothness of the loss, and (e) is the cauchy inequality. Finally choosing the second step size to be $\gamma=O(\Gamma M_t \eta)$ makes the last two terms term cancel out, for any random example $\hat{z}$. Note that $\Gamma \geq tr(cov[z_t])$ where $cov[z_t]$ is the covariance of the received examples, and $M_t$ is the smoothness parameter of the loss function.
\end{proof}

\section{Proof of Theorem 3}

\begin{theorem}[Theorem 3 restated]
\label{RandFerror}
Given a pairwise Mercer kernel $k_{(x,x')} :=k((x,x'),(\cdot,\cdot))$ defined on $\mathcal{X}^2\times\mathcal{X}^2$. Let $\ell(w,z,z')$ be convex loss that is Lipschits continuous with constant $G$. Then for any $w^* = \sum_{i,j\neq i}^Ta_{i,j}^* k_{(x_i,x_j)}$, and random Fourier features number $D$ we have the following with probability at least $1-2^8\left(\frac{\sigma Diam(\mathcal{X})}{\epsilon}\right)\exp(-D\epsilon^2/(4d+8))$,
\begin{align}
    \sum_{t=2}^T L_t(\bar{w}^*)  - \sum_{t=2}^T L_t(w^*) \leq 
 G T \|w^*\|_1 \epsilon
\end{align}
where $ \|w^*\|_1 = \sum_{i,j\neq i}^T|a^*_{i,j}|$, and $D= \Omega\left(\frac{d}{\epsilon^2} \log\frac{\sigma diam(\mathcal{X})}{\epsilon} \right)$.
\end{theorem}
The proof depends on the main claim of \citep{rahimi2007random}, we rewrite the results here.
\begin{lemma}[Claim 1 \citep{rahimi2007random}]
Given $x,x'\in \mathcal{X}$, and kernel $\mathcal{G}$ defined on $\mathcal{X\times X}$, we have the following bound on the error of the approximation $\hat{\mathcal{G}}$ using $D$ random Fourier features.
\begin{align}\label{lemma:Rahimi}
\mathbb{P}\left(    |\hat{\mathcal{G}}(x,x')   - \mathcal{G}(x,x')  |\geq \epsilon  \right) \leq 2\exp\left(\frac{-D\epsilon^2}{2}\right)
\end{align}
\end{lemma}

The pairwise kernel defined as,
\begin{align}\label{eq:pairwiseKernel}
 &   k(x_1,x_2,x_1',x_2') = \mathcal{G}(x_1,x_1') + \mathcal{G}(x_2,x_2') \\\nonumber
 &    - \mathcal{G}(x_1,x_2') - \mathcal{G}(x_2,x_1') = \langle \mathcal{G}_{x_1} - \mathcal{G}_{x_2}, \mathcal{G}_{x_1'} - \mathcal{G}_{x_2'} \rangle_\mathcal{G}
\end{align}
The following Corollary bounds the error of the pairwise kernel defined in equation (\ref{eq:pairwiseKernel}) using main theorem in \citep{rahimi2007random}.
\begin{corollary}\label{RFerror1}
Given $x_1,x_2,x_1',x_2'
\in\mathcal{X}$, and pairwise kernel $k$ defined on $\mathcal{X}^2\times\mathcal{X}^2$, the random Fourier estimation of the kernel has an error bounded with probability at least $1-8\exp{\frac{-D^2\epsilon}{2}}$ as follow,
\begin{align}\nonumber
   \sup_{x_i,x_t\in\mathcal{X}}   |\hat{k}_{(x_1,x_2)}(x_1',x_2')   - {k}_{(x_1,x_2)}(x_1',x_2')  | \leq \epsilon  \end{align}
\end{corollary}
Now we are ready to prove theorem 3, where we start by difference of the loss between the model $\bar{w}^*$ and $w^*$,
\begin{proof}
\begin{align}\nonumber
 \sum_{t=2}^T L_t(\bar{w}^*)  - \sum_{t=2}^T L_t(w^*)&=  \sum_{t=2}^T\frac{1}{t-1}\sum_{i=1}^{t-1}\ell(\bar{w}^*,z_t,z_i) -\sum_{t=2}^T\frac{1}{t-1}\sum_{i=1}^{t-1}\ell({w}^*,z_t,z_i)\\\nonumber
 &\overset{a}{\leq}\sum_{t=2}^T\frac{1}{t-1}\sum_{i=1}^{t-1}|\ell(\bar{w}^*,z_t,z_i) -\ell({w}^*,z_t,z_i)|\\
 &\overset{b}{\leq} \sum_{t=2}^T\frac{1}{t-1}\sum_{i=1}^{t-1} G \|\bar{w}^* - {w}^* \|_{L^2(\rho)}
 \end{align}
 where inequality (a) implements the triangle inequality, and inequality (b) applies assumption 1 and $\|\cdot\|_{L^2(\rho)}\geq \|\cdot\|_2$. 
 \\
 Using the Representer theorem and assumption \ref{ass:Kernel}, we assume that the space $\mathcal{H}$ and $\bar{\mathcal{H}}$ are dense in $L^2(\rho)$ (the space of square integrable function under the probability distribution of the received examples), then we can approximate any function in $\mathcal{H}$ by a function in $L^2(\rho)$, i.e. without loss of generality we assume that $\bar{w}^* = \sum_{j=1,k\neq j}^{t-1} \alpha_{j,k}^* \hat{k}_{z_t,z_i}$, then we have,
 \begin{equation}
     \|\bar{w}^* - {w}^* \|_{L^2(\rho)} =  \|\sum_{j=1,k\neq j}^{t-1}\alpha_{j,k}^* (\hat{k}_{z_t,z_i} - {k}_{z_t,z_i} ) \|_{L^2(\rho)}
 \end{equation}
 and using the triangle inequality we have,
 \begin{align}\nonumber
  \sum_{t=2}^T L_t(\bar{w}^*)  - \sum_{t=2}^T L_t(w^*) &{\leq}\sum_{t=2}^T \sup_{x_i,x_t\in\mathcal{X}} G\sum_{j=1,k\neq j}^{t-1} \|\alpha_{j,k}^* (\hat{k}_{z_t,z_i} - {k}_{z_t,z_i} )\|_{L^2(\rho)}\\
 &\overset{c}{\leq}\sum_{t=2}^T G\sum_{j=1,k\neq j}^{t-1} |\alpha_{j,k}^*| \epsilon \leq 
 G T \|w^*\|_1 \epsilon
 \end{align}
where inequality (c) implements corollary \ref{RFerror1}, and $ \|w^*\|_1 = \sum_{i,j\neq i}^T|a^*_{i,j}|$. \\
This completes the proof.
\end{proof}

\end{document}